\documentclass[letterpaper]{article} 
\usepackage{aaai24}  
\usepackage{times}  
\usepackage{helvet}  
\usepackage{courier}  
\usepackage[hyphens]{url}  
\usepackage{graphicx} 
\usepackage{natbib}  
\usepackage{caption} 
\usepackage{algorithm}
\usepackage{algorithmic}

\usepackage{xcolor}
\usepackage{amsmath}
\usepackage{booktabs}
\usepackage{placeins}
\usepackage{amsthm}
\usepackage{newfloat}
\usepackage{listings}

\newtheorem{theorem}{Theorem}

\newtheorem{lemma}{Lemma}
\newtheorem{corollary}{Corollary}[theorem] 
\DeclareMathOperator*{\argmax}{arg\,max}

\DeclareCaptionStyle{ruled}{labelfont=normalfont,labelsep=colon,strut=off} 
\floatstyle{ruled}
\newfloat{listing}{tb}{lst}{}
\floatname{listing}{Listing}
\usepackage{xcolor}

\title{Investigating the Benefits of Nonlinear Action Maps in Data-Driven Teleoperation}
\author {
    Michael Przystupa\textsuperscript{\rm 1,\rm 2},
    Gauthier Gidel \textsuperscript{\rm 3},
    Matthew E.~Taylor \textsuperscript{\rm 2,5}, \\
    Martin Jagersand \textsuperscript{\rm 2},
    Justus Piater \textsuperscript{\rm 4},
    Samuele Tosatto \textsuperscript{\rm 4}
}
\affiliations {
    \textsuperscript{\rm 1}przystup@ualberta.com
    \textsuperscript{\rm 2}University of Alberta, 
    \textsuperscript{\rm 3}Universite de Montreal, 
    \textsuperscript{\rm 4}University of Innsbruck,
    \textsuperscript{\rm 5}Alberta Machine Intelligence Institute (Amii)
    
}

\usepackage{bibentry}

\begin{document}

\maketitle

\begin{abstract}
As robots become more common for both able-bodied individuals and those living with a disability, it is increasingly important that lay people be able to drive multi-degree-of-freedom platforms with low-dimensional controllers. One approach is to use state-conditioned action mapping methods to learn mappings between low-dimensional controllers and high DOF manipulators -- prior research suggests these mappings can simplify the teleoperation experience for users. Recent works suggest that neural networks predicting a local linear function are superior to the typical end-to-end multi-layer perceptrons because they allow users to more easily undo actions, providing more control over the system. However, local linear models assume actions exist on a linear subspace and may not capture nuanced actions in training data. We observe that the benefit of these mappings is being an odd function concerning user actions, and propose end-to-end nonlinear action maps which achieve this property. Unfortunately, our experiments show that such modifications offer minimal advantages over previous solutions. We find that nonlinear odd functions behave linearly  for most of the control space, suggesting architecture structure improvements are not the primary factor in data-driven teleoperation. Our results suggest other avenues, such as data augmentation techniques and analysis of human behavior, are necessary for action maps to become practical in real-world applications, such as in assistive robotics to improve the quality of life of people living with w disability.   

\end{abstract}

\section{Introduction} \label{sec:introduction}

\begin{figure}
\includegraphics[width=0.5\textwidth,,trim=10mm 2mm 2mm 0mm, clip]{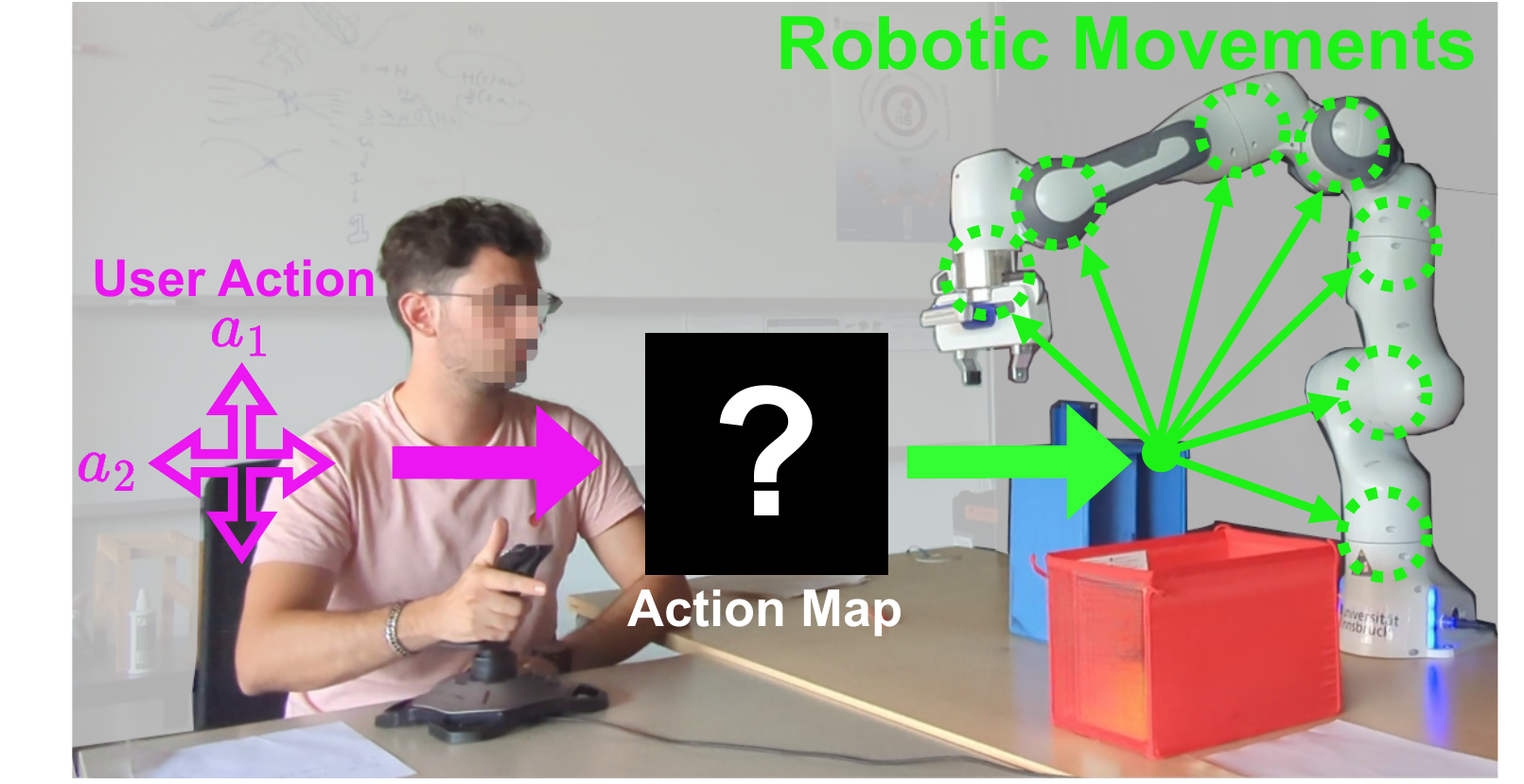}
  \caption{User teleoperating robot with data driven interface.}
  \label{fig:teaser}
\end{figure}

The choice of coordinate system in teleoperation can dramatically affect the difficulty of robotic manipulation problems. For example, wiping a table is often easier to control directly in the corresponding 2D Cartesian plane with respect to a table, as opposed to the original N-dimensional joint space of the robotic arm. This observation motivates typical \textit{mode switching} mappings, in which users control subsets of the robot pose at any time \cite{newman2022harmonic}. These mappings are the de facto scheme in assistive robotic systems and are often intuitive because user inputs map to the coordinate system of the robot's end-effector. Such interfaces are valuable for users with limited mobility, restricting their teleoperation interfaces' available degrees of freedom.

Unfortunately, results suggest that mode switching can be mentally taxing, involving up to 30 -- 60 mode switches to complete everyday tasks like opening doors \cite{herlant2016timeoptimalswitching,routhier2010usability}. Such mappings are challenging to scale to complex teleoperation scenarios (e.g., parallel robotics). Mode switching's limitations have motivated research on data-driven action mapping algorithms that project low-dimensional actions into high-dimensional control commands \cite{losey2021learnlatentLONGpaper,losey2019controlling} and are learned with demonstration data. Figure~\ref{fig:teaser} provides  a visualizatoin of action maps at deployment.

Early work focused on linear action maps to reconstruct the relevant action space. Linear methods are motivated by results modeling hand grasps, where one can approximate 80\% hand-poses with only the first two loading vectors found via principal component analysis \cite{santello1998posturalhands}. Researchers have previously applied linear action maps in robotic control algorithms \cite{artemiadis2010emgbasecontrol,ciocarlie2009handposture,odest2007twoDsubspace,matrone2012realtimemyoelectric}.

Recent work has proposed \textit{state conditioned} action maps which adapt high-dimension commands given task-context information. A conditional autoencoder (CAE) learns to reconstruct high-dimension manipulations from states and a latent variable that maps a user's input. Several aspects have been studied, including user experience \cite{jeon2020SharedAW,Li2020LearningUM}, efficient data collection \cite{mehta2022LatentActionwoHuman}, and complex data conditioning (e.g., images or text) \cite{karamcheti2022LILa,karamcheti2021learningvisuallyguided}.  Beyond assistive robotics, researchers have loosely studied CAE action maps in reinforcement learning (RL) for sample efficiency \cite{allshire2022laser,chandak2019learningactionrepr}, offline RL \cite{zhou2020PLAS}, and learning hand-synergies \cite{he2022discoveringsynergies}.

A limitation of the CAE approach is tying the quality of the demonstration data with the functional teleoperation properties \cite{losey2019controlling}. Relying on data to ``discover" functional teleoperation properties is naive because human teleoperation priors are already known. \citet{Li2020LearningUM} exploits this observation to propose regularizer terms to help learn a remapping model between an action map and the user's preferences. \citet{przystupa2023SCL} investigated enforcing \textit{teleoperation reversibility}, where users expect symmetrically opposite actions to undo prior movements. The authors proposed \textit{soft reversibility} to bound errors when reversing actions. State-conditioned maps vary over time with fixed user inputs, and performing the opposite actions does not necessarily map to the inverse movements.

In particular, \citet{przystupa2023SCL} propose learning a \textit{local linear} model instead of the typical multi-layer perceptron. Their model uses the state to generate a hypothesis linear function that combines user actions into robot commands. Their proof of soft-reversibility relies on linear functions being odd, but this core assumption does not necessarily mean that a function must be linear. One could construct odd non-linear functions that may more faithfully reconstruct actions, which could improve user experience.

These observations motivate our investigation of whether alternative nonlinear models achieve better user performance in teleoperation. We study this by proposing a regularization approach to imbue a typical end-to-end decoder with soft reversibility and extending the architecture proposed by \citet{przystupa2023SCL}. We evaluate both in simulation and with a user study. 

We hypothesize that \textit{any soft-reversible action map, and therefore an odd function, will lead to similar user experiences in a single robot arm setting}. This hypothesis is motivated by analysis of the first-order Taylor approximation of a function. As our results will show, nonlinear action maps behave similarly to their linearization for relevant domains of the user actions, and a user study demonstrates limited statistical differences between models during a teleoperation task both in user subjective metrics and objective success rates. 

The remainder of this paper outlines our findings in investigating this hypothesis. We include a background section to summarize the relevant aspects of action mapping learning and modeling assumptions we make. We then explain our proposed alternatives for learning odd action mappings concerning actions and show experiments validating our hypothesis. We further include soft-reversibility results extending the work of \citet{przystupa2023SCL} in the appendix. These contributions are not the focus of this paper.

Our work is a necessary step forward in developing data-driven interfaces by directly assessing model design. Many works have tackled issues of model input features and collection \cite{mehta2022LatentActionwoHuman,karamcheti2021learningvisuallyguided,karamcheti2022LILa,jeon2020SharedAW,Li2020LearningUM}, but the impact of model structure has received limited attention \cite{przystupa2023SCL}. Local Linear models are easier to analyze, so if they are sufficient for teleoperation, this paves the way for accessibility of existing system analysis tools \cite{Lavalle2006PlanningAlgs}. Applying such tools could guarantee properties previously proposed as necessary in data-driven teleoperation systems but have only been loosely empirically demonstrated as being enforced \cite{losey2019controlling,losey2021learnlatentLONGpaper}. As data-driven teleoperation involves human users, concrete guarantees of how the learning system will behave are necessary for any ethical approval of real-world use outside academic research. 

\section{Background}\label{sec:background}
Previous works on data-driven teleoperation frame action mappings as Markov decision processes (MDP) where each action occurs at discrete time steps \cite{losey2021learnlatentLONGpaper}. Instead, we describe action mappings as a controllable dynamical system over continuous time. Our theoretical analysis in the appendix utilizes this framework to interpret the gradient flow induced by user actions. We then discuss the general overview of a typical deep learning action mapping framework and describe soft reversibility. These latter topics are the most important for our work described in this paper. 

\subsection{Dynamical Control Systems}\
In the teleoperation settings, we essentially have a first-order ordinary differential system of equations:
\begin{equation}\label{eqn:dynamicalsystem}
    \dot{x}(t) = f_{\theta}(x(t), a(t)),
\end{equation}
where $x(t) \in \mathbf{R}^{d}$ are the state variables, $a(t) \in R^n$ are user actions, time $t \in \mathbf{R}^+$, and $f_{\theta}: X \times A \rightarrow \dot{X}$ is the learned mapping from states $X$ and user actions $A$ to robot velocities $\dot{X}$. $f$ implicitly depends on time through $x$ and $a$ and dimensions $n \ll d$.

We assume that for teleoperation, the dynamical system evolves following Euler's method:
\begin{equation} \label{eqn:euler}
    x_{T}(k+ 1) = x_{T}(k) + \nu f(x_{T}(k), a_{T}(k))
\end{equation}
where we introduce variables with a fixed discretization duration $T$ steps (e.g. $x_T(k) = x(k/T)$) and $\nu \in \mathbf{R}^+$ is the update rate during integration. We consider these modeling assumptions reasonable as the teleoperation context is a closed-loop system, where the trajectory is determined online by the internal policy $\pi(o_{h}) = a$ of a human operating unknown $o_h$ world observations.

\subsection{Action Mapping Methods}

Action mapping methods learn intermediate models that some agent (a human in teleoperation) interacts with a robotic system. We drop the notation for variables dependent on time for succinctness and because it is irrelevant for learning action mapping models. We assume access to the dataset $D = \{(x, \dot{x}) : (x, \dot{x}) \sim p(x'|x, \dot{x})\pi(\dot{x}|x) \}$ generated by some policy $\pi$ for the task of interest to control with dynamics $p(x' | x, \dot{x})$. In practice, a user can collect demonstration through kinesthetic demonstration or unsupervised methods \cite{mehta2022LatentActionwoHuman}.

A user can then train a CAE model with the data by optimizing some variation of the following loss $L(x, \dot{x}, f, g)$:
\begin{equation}
    L^{\text{\it{recon}}}(\dot{x}, f(x, g(x, \dot{x}))) + L^{\text{\it{reg}}}(f, g, x, \dot{x}).
\end{equation}
The mean square error is sufficient for action reconstruction in practice $L^{\text{\it{recon}}} = \frac{1}{2}||\dot{x} - f(x, g(x, \dot{x}))||^2_2$. Previous work has considered different regularizers ($L^{\text{\it{reg}}}$), including the Kullback-Liebler divergence or latent dynamics models \cite{allshire2022laser,losey2019controlling}. Our results will show that regularizers that encode desirable teleoperation properties can be beneficial in giving users control of the mapping function. 

The typical CAE model includes neural encoder $g_\theta(x, \dot{x}) = a$ that generates the low-dimensional actions $a$ and a decoder $f_\theta(x, a) = \dot{x}$ that reconstructs the original robotic commands. An important distinction is the notion of end-to-end versus hyper-linear decoders. End-to-end encoders rely on the additive interactions between $x$ and $a$ by concatenating both variables and providing them as input to a multi-layer perceptron (MLP). A hyper-linear decoder is a form of \textit{hypernetwork} \cite{hypernetworksHa2016} that decomposes action and context interactions; given the robotic context $x$,  a local linear function is predicted to reconstruct actions: $f_\theta(x, a) = h_\theta(x)a$, where $h_\theta(x) = H \in R^{d \times n}$. Previous work suggests these give more control to a user, unlike end-to-end encoders, which rely on the neural network to encode any properties for teleoperation \cite{przystupa2023SCL}. 

\subsection{Soft Reversibility}

Soft reversibility is a property of actions mapping algorithms proposed by \citet{przystupa2023SCL}. The key observation is that for three states $x_i$, $x_{i+1}$, $x_{i+2}$ following Equation~\ref{eqn:euler}, if $a_{i+1} = -a_i$ then $\| x_i - x_{i+2}\| < \| x_{i+1} - x_i\|$. The result models the prior assumption that humans expect joystick actions to \textit{undo} prior actions when pushing in the corresponding opposite direction. A limitation of this result is that it assumes discrete steps and only reverses a single triplet of state action pairs. We extend these results in the appendix to account for trajectories of actions.
\section{Nonlinear State Action Maps}\label{sec:methods}

This section describes two ways to learn action mappings that lead to odd functional behavior -- as required for soft-reversibility -- that do not assume a linear function. The first approach adds regularizers to an end-to-end decoder to encourage the model to be odd in the actions and maps zero actions to zero movements. This approach aligns with the typical nonlinear black box by controlling the model through the loss function. The second method generalizes the hyper-linear hypernetwork through additional local action map assumptions. We enforce the upper bound on both approaches Lipschitz constant with the learning algorithm of \citet{gouk2020regularisation} and include pseudo code in the appendix.

\subsection{Regularizing Teleoperation Behavior}

A natural means to enforce teleoperation properties in a neural network is the specification of the loss function. Through this approach we propose the following loss function $L^{reg}(g, f,  x, \dot{x}) $:
\begin{equation}
    ||(-\dot{x}) - f(x, -a)||^2_2 + ||f(x, 0)||^2_2 + ||g(x, 0)||^2_2.
\end{equation}
Our proposed regularizer achieves two behaviors: (1) enforcing odd functional behavior to enable reversing and (2) encouraging that zero actions map to zero outputs. We include this regularizer in each mini-batch by adding $\dot{x} = 0$ and inverse action $-a$ with the state inputs in the batch.

Our proposed regularizer is advantageous compared to alternative approaches due to its simplicity. It is always applicable during training because it only requires predictions generated in the CAE framework. The loss terms of \citet{Li2020LearningUM} need kinematics knowledge to apply, and other works found they did not improve CAE behavior in simulated reversibility \cite{przystupa2023SCL}. Our loss function also defines expected latent space behaviors. In contrast, reported regularizer terms only compress the action space \cite{losey2019controlling} or constrain a black-box dynamics model \cite{allshire2022laser}. One limitation of this approach is that the odd function behavior is proximal, introducing residual errors in the predictions. Our theoretical results in the appendix show how this can be problematic. 


\subsection{State Conditioned Nonlinear Maps}
An alternative to changing the loss function is decomposing the interactions between context conditioning and the action inputs model. The hyper-linear decoders discussed are an example of this, where the action inputs interact with a function conditioned on state, as opposed to directly interacting with the modifications of the state in the model. The trade-off with this approach is that it assumes all actions for a given context exist on a linear manifold. We propose \textit{State Conditioned Nonlinear Maps} (SCN), which generalize the hyper-linear model to allow for several layers which expand the function class to nonlinear odd functions, which allows them to be soft-reversible. The model inductive bias means it can be incompatible with learning certain functions (e.g., asymmetrical functions of the inputs).

\subsubsection{Tensor Layers}
One approach for hyper-linear layers is to reshape a vector output into a matrix. We point out that this reshape operation is not necessary, and one can instead interpret the hyper-linear layer as a 3D tensor product: 
\begin{equation}
    W^i(x) =  H^{i} \otimes \phi_\theta(x)  + B^{i}.
\end{equation}
Where $W(x) \in R^{h\times n}$ is the resulting matrix after tensor product $\otimes$ between $H \in \mathbf{R}^{h \times w \times n}$ and nonlinear projection of $\phi_{\theta}(x) \in \mathbf{R}^w$ states with matrix bias $B^{i} \in \mathbf{R}^{h \times n}$. An SCN model is a composition of these tensor layers, where a single hidden layer and activation function $\sigma$ looks like:
\begin{equation}
    \hat{x} = W^{2}(x)\sigma(W^{1}(x)a).
\end{equation}
An important observation is that this model \textbf{does not} have any additional bias terms concerning the user actions. Removing this bias term is the simplest means of enforcing the model to be odd in the user actions.

\subsubsection{Architecture Constraints}

Unfortunately, additional architecture constraints are necessary with SCN. Tensors can quickly increase the model parameter counts if set naively, and keeping the same model size as end-to-end models means fewer neurons per hidden layer. We find sharing the features of $\phi_{\theta}(x)$ between layers sufficient. 

Furthermore, we cannot use arbitrary activation functions $\sigma$ because of the need for an end-to-end odd function in actions. Enforcing this requires symmetric activation functions around their inputs, $\sigma(x) \in [-c, c]$ for $c \in R^+$, including hyperbolic tangent, sinusoidal, or otherwise modified activation functions to make similar constraints.

Finally, an additional constraint of soft-reversibility is that $f$ is Lipschitz. As previously mentioned, we can use the algorithm of \citet{gouk2020regularisation} to bound a network's upper bound on the Lipschitz constant. As Tensor layers are 3D, a few modifications are required to correctly specify the Lipschitz upper bound. The appendix contains a procedure to control the Lipschitz bound of tensor layers.

\begin{figure}
    \centering
    \includegraphics[width=0.4\textwidth]{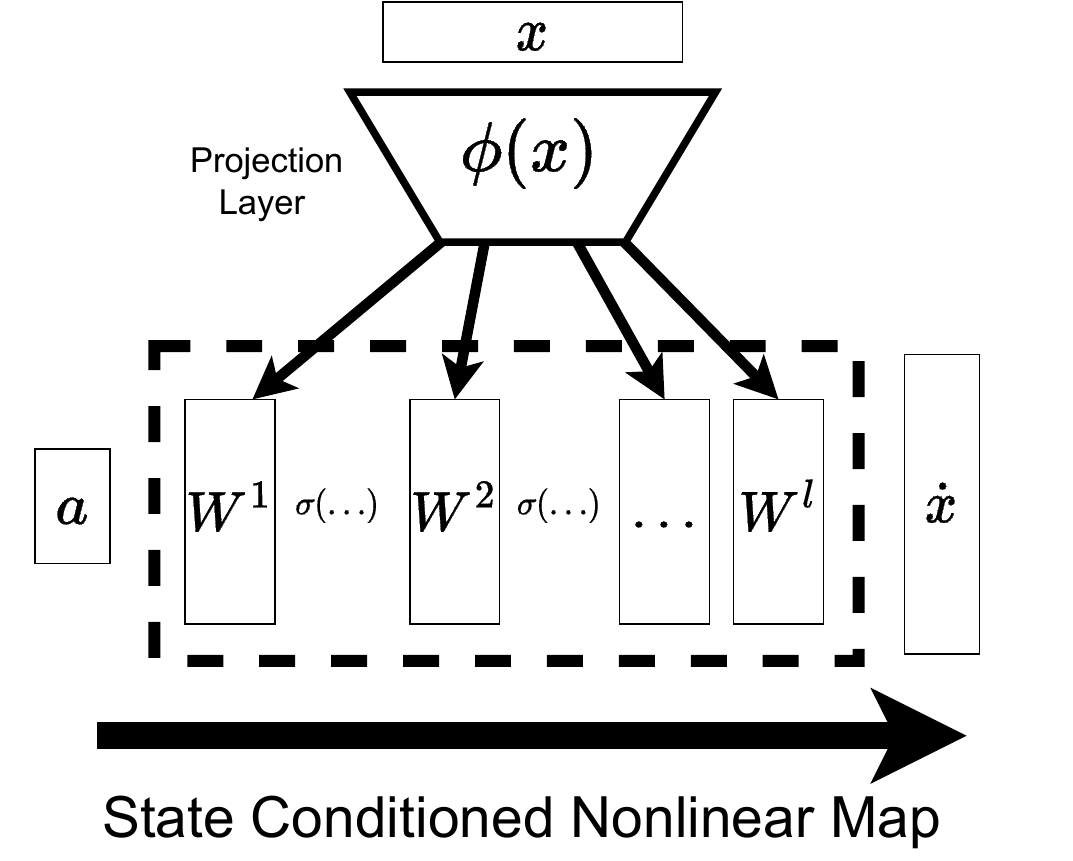}
    \caption{Diagram of neural architecture for action maps.}
    \label{fig:tensormaps}
\end{figure}

\subsection{Limitations of Nonlinearity}
Although we propose two means of imbuing soft reversibility into end-to-end nonlinear neural networks, a pertinent question is whether nonlinear action maps have any benefit from a user perspective. Specifically, we hypothesize that \textit{there will be no statistically significant difference between our proposed nonlinear action map methods and local linear models}. One reason for this hypothesis is that the hyper-linear decoders are already soft reversible, giving users significant control over the system. Our second motivation stems from the first-order Taylor approximation around $a_0 = 0$:
\begin{align*}
    f(x,a) &\approx f(x, a_0) + J^{a}(s, a_0)(a - a_0) \\
           &= f(x, 0) + J^{a}(x, 0)a 
\end{align*}
With our proposed modifications, the Taylor series approximation with both proposed methods means $f(x, 0) = 0$ or $f(x, 0) \approx 0$, reducing the model to being dependent solely on the Jacobian centered at taking no action, or in other words: a local linear model! We evaluate this consideration in our experiments. 

\section{Experiments} \label{sec:expectiments}

\begin{figure}
    \centering
    \includegraphics[width=0.45\textwidth]{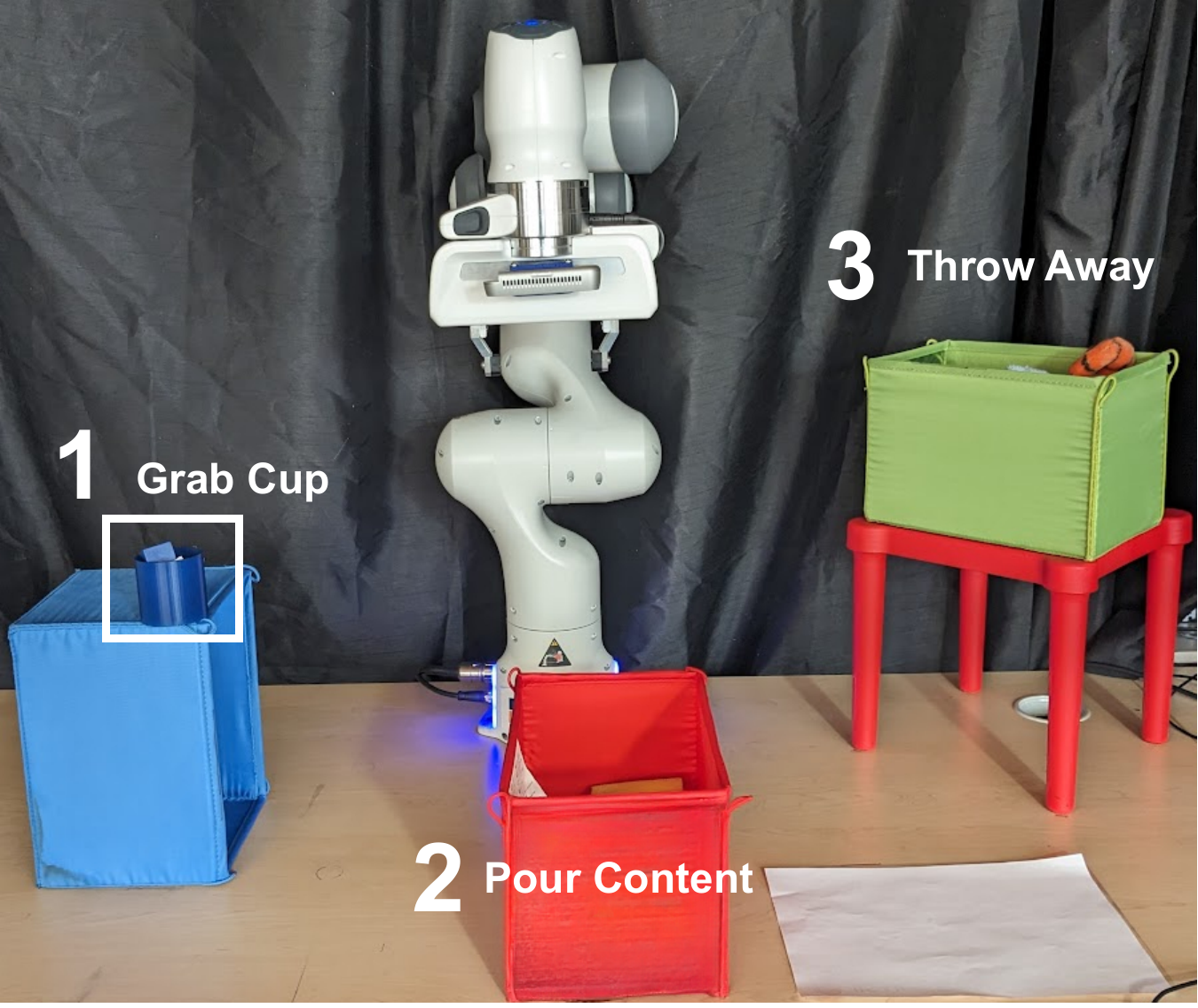}
    \caption{User study environment for data-driven teleoperation experiments. Users had to pick up the cup on the blue box, dispense it's contents into the red container and finally dispense the cup in the green box on the opposite end of the testing environment. }
    \label{fig:userstudysetup}
\end{figure}

\begin{table*}[htp]
\centering
{\normalsize
\begin{tabular}{@{}ccccccc@{}}
\toprule
\multicolumn{1}{l}{} & Bottle to Shelf             & Open Box                    & Plates Standing to Lying    & Pouring                     & Scoop                       \\ \midrule
SCL                  & $\mathbf{-5.727 \pm 0.010}$ & $-5.981 \pm 0.020$          & $-5.451 \pm 0.010$          & $-5.395 \pm 0.005$          & $-5.832 \pm 0.015$          \\
AE                   & $-5.546 \pm 0.007$          & $-5.761 \pm 0.002$          & $-5.458 \pm 0.002$          & $-5.303 \pm 0.027$          & $-5.702 \pm 0.005$          \\
AE (Loss)            & $-5.420 \pm 0.171$          & $-5.264 \pm 0.219$          & $-5.312 \pm 0.132$          & $-5.149 \pm 0.107$          & $-5.263 \pm 0.148$          \\
Tensor               & $-5.609 \pm 0.098$          & $-5.996 \pm 0.012$          & $\mathbf{-5.500 \pm 0.089}$ & $\mathbf{-5.480 \pm 0.010}$ & $-5.717 \pm 0.117$          \\
Tensor (Loss)        & $-5.720 \pm 0.006$          & $\mathbf{-5.998 \pm 0.024}$ & $-5.436 \pm 0.011$          & $-5.421 \pm 0.009$          & $\mathbf{-5.843 \pm 0.005}$ \\ \bottomrule
\end{tabular}
} 
\caption{ Log Base 10 Mean Square Test Error, reproducing actions with different state conditioned action maps. Lower values are best performance. We find the introduced regularizer terms in the user study tend to worsen reconstruction for models but the trained models are symmmetrical in the actions. } \label{tab:actionreconstruction}
\end{table*}

In this section, we report experiment results investigating the value of the proposed nonlinear action mapping approaches. We compare against previously proposed state-conditioned linear (SCL) maps as a baseline \cite{przystupa2023SCL}. SCL is a local linear map that transforms the matrix predictions with the gram-schmidt process at deployment. In this section, we refer \textit{AE} and \textit{SCN}, respectively for the MLP or tensor layer models and use $\textit{Loss}$ when discussing models trained with the proposed regularizers.

Our first interest is evaluating the quality of each model with empirical metrics in simulation. Specifically, we measure each model's prediction accuracy and the effect of the local linear approximation away from $a = 0$. We then conduct simulated teleoperation experiments to investigate how well the models can track trajectories. We close with a user study to evaluate our hypothesis on non-linearity's benefit with able-body participants, testing whether statistical significance exists in subjective and objective metrics. 



 In all experiments, we use the same hyperparameters for training each system. Each model had three linear layers with tanh activations. SCN models had 32 units in the tensor layers and 48 in the state projection layer, whereas SCL and AE models had 256 neurons per layer.  We select these layer parameters to keep the total parameters similar between models, but we note that SCN models had fewer parameters (approximately 65,000 vs 70,000 for other models). We train each model for either 1000 epochs for regression and simulated teleoperation experiments and 500 epochs for the user study experiments. We trained all models with mini-batches of 256 with the Adam optimizer with a learning rate 0.001. 
 
\subsection{Action Map Analysis}
This subsection reports results analyzing the quality of our proposed models. The focus is on how well the proposed action mapping algorithms perform in simulation under ideal scenarios. We consider these unrealistic because we can predict optimal actions under the model via the trained encoder or apply optimization techniques that may not accurately reflect human behaviors (e.g., simulating several trajectories under the action map).

\subsubsection{Model Regression Quality }

Our first interest is analyzing the predictive quality of our model. We reconstruct publicly available demonstration trajectories from the work of \citet{sayantan2023contlearndemo}.  Each data set consists of 10 trajectories with 1000 samples each (10000 interaction tuples). We treated the finite differences between every 3rd sample as the actions because we found finite differences easier to control in our preliminary user study experiments. For each demonstration in the data set, we use 90\% of interaction tuples for training, 5\% for validation, and 5\% for testing. We report test mean square error in Table~\ref{tab:actionreconstruction}. Our results suggest that SCN models are generally better at reconstructing actions, whereas AE models are less accurate and seem negatively impacted by our regularizer loss functions. 

We further compare the effect of the linearization of each model on the output prediction. We compare the linear approximation against the model prediction on the test data set $||f(x, a) - f(x, 0) - \nabla_a f(x, 0)a||^2_2$. We increase the magnitude of $||a||$ at each testing state to see how the approximation varies in Figure~\ref{fig:linearizationRegressRes}, which suggests our models are close to their linearization in action norm ranges considered for our user study, where \textit{Study} was on the user study dataset. Note that the regularizer seems to make SCN and AE models more nonlinear further from the approximation.
\begin{figure}[h] 
        \centering
        \includegraphics[trim=0mm 0mm 0mm 0mm, clip=true, width=0.45\textwidth]{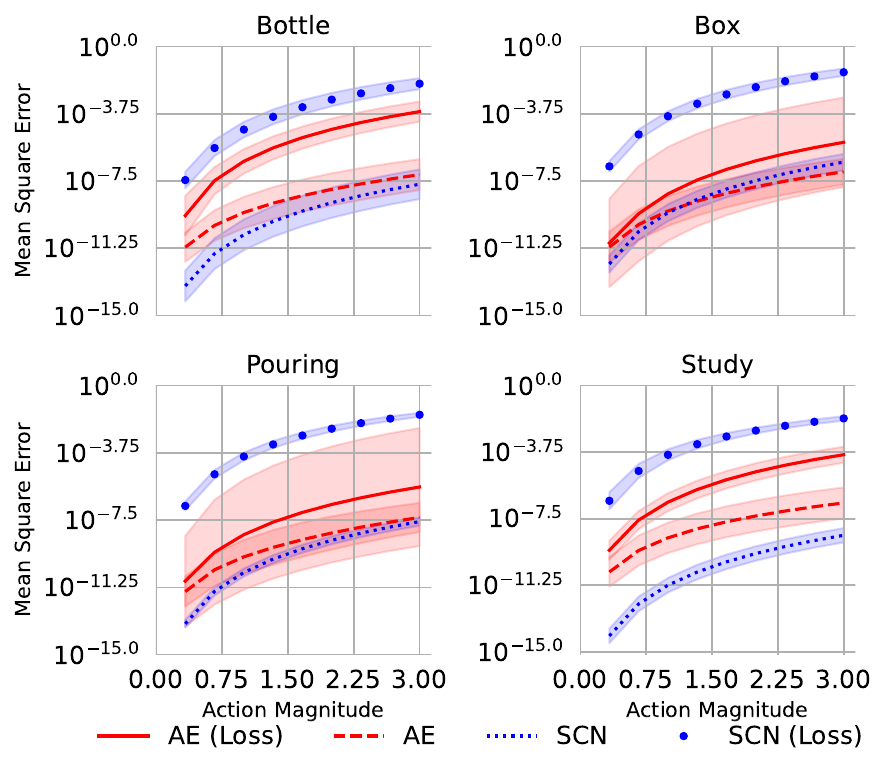}
        \caption{Difference between prediction and linearization as action as magnitude increases.}
        \label{fig:linearizationRegressRes}
\end{figure}

\subsubsection{Simulated Teleoperation}

In our last set of simulation experiments, we run an artificial teleoperation experiment to use the action map to reach a specified target $x^\star$ in the test data sets from \citet{sayantan2023contlearndemo}. The actions were generated greedily with the following approach:
\begin{equation}
    a_t = \min_{a \sim N(0, I)} ||x^{\star} - x_t - \nu f(x_t, a_t)||^2_2
\end{equation}
We found this could be unstable in practice due to local minima, and we mitigate this effect with via-points in a test trajectory to improve the accuracy of simulated teleoperation. Mean performances are in Figure~\ref{fig:linearizationcomparison} across different demonstration data sets with their optimal action norm size. Figure~\ref{fig:linearizationcomparison} also shows the effect of action norms size for both AE (Loss) and SCN in terms of getting closer to the target states.

\begingroup
\begin{figure}[ht] 
\vspace{0.5in}
\hspace{0.00in}
    \begin{minipage}[htp]{0.45\textwidth}
    \centering
        \includegraphics[width=0.75\textwidth,trim=0mm 15 0 0, clip]{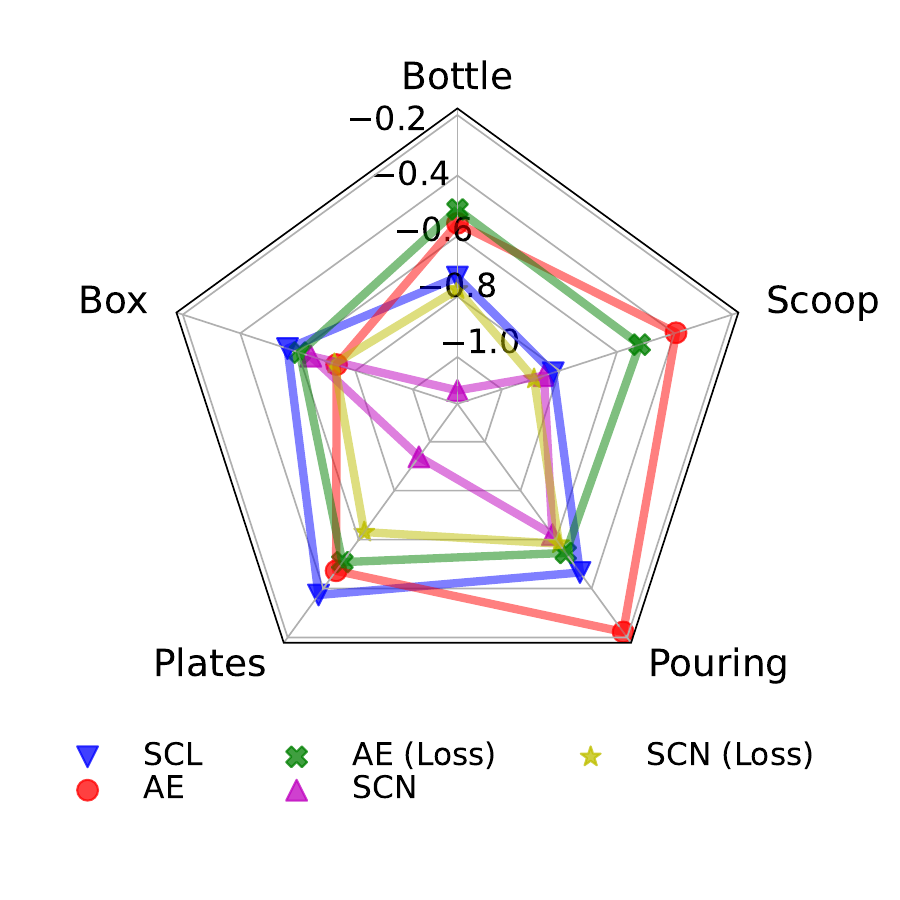}
        \caption*{Model Comparisons}
    \end{minipage}
    \begin{minipage}[htp]{0.45\textwidth}
    \centering
        \includegraphics[width=0.75\textwidth,trim=0mm 15 0 0, clip]{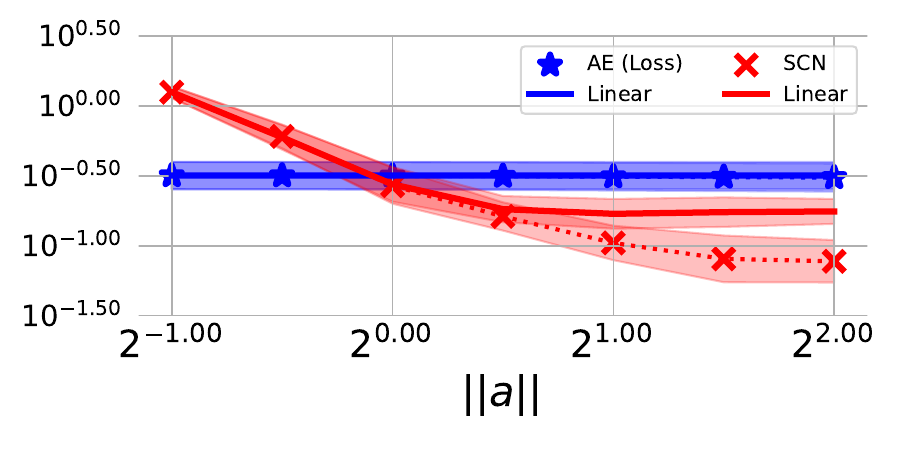}
        \caption*{Bottle To Shelf Demonstrations }
    \end{minipage}
\caption{Simulated Teleoperation Experiments comparing against linearized models. We show best simulation results with tuned magnitude and affect of action norm on reaching targets with AE and SCN models. }
\label{fig:linearizationcomparison}
\end{figure}
\endgroup

\subsection{User Study Results}

\begin{figure}[htp] 
    \centering
        \includegraphics[width=0.4\textwidth,trim=0 5 0 0, clip]{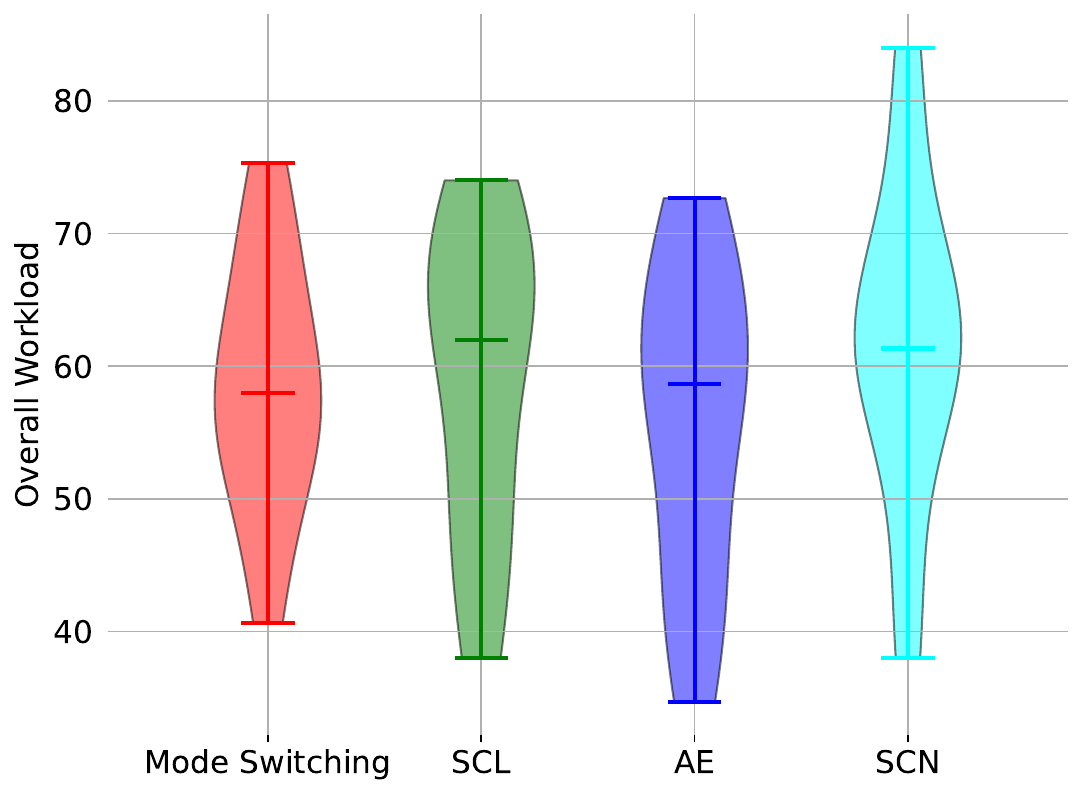}
        \caption{Nasa Workload Index between systems during user study. Between all system used by participants we did not find any statistical significants with p-value 0.05 by the Posthoc Dunn mean test. }
        \label{fig:nasaworkload}
\end{figure}

\begin{figure}[htp] 
        \centering
        \includegraphics[trim=0mm 0mm 5mm 5mm, clip=true, width=0.5\textwidth]{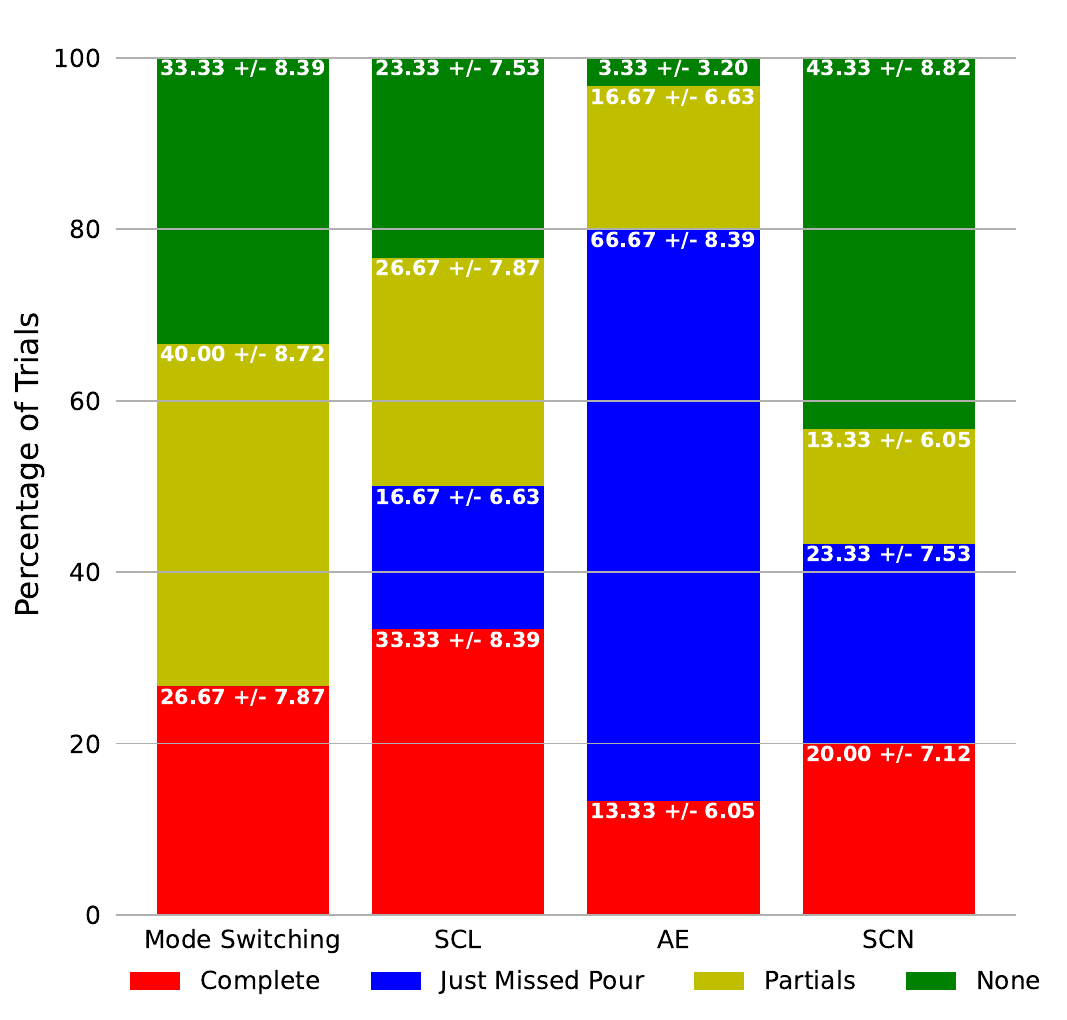}
        \caption{Aggregate success rates for different classifications of the user study. A full completion were when participants successfully picked up the cup, poured contents into the container and threw cup away.}
        \label{fig:completion Rates}
\end{figure}

To validate our hypothesis on the impact of model structure, we perform a user study to complete a sequential multi-stage task. Figure~\ref{fig:userstudysetup} shows the experimental set-up. We asked participants to pick up a cup from the top of a box (blue), empty its contents in the middle container (red), and throw the cup away in the corner on the opposite side of the testing space (green). This task transitioned between multiple aspects of the Cartesian pose for each sub-task. For each model considered, participants completed four trials, where the first two were two minutes (training trials), and the latter two were three minutes (testing trials). We had 15 participants (13 male, two female) ages 20 - 55.

Each participant first controlled the robot with a mode-switching mapping in the experiments. Our mode-switching system had six modes: the first three corresponded to translation control, and the last three corresponded to orientation. We chose to have mode switching first because our main interest was comparing the data-driven interfaces. These trials helped familiarize users with the joy-stick interface and psychologically establish baseline expectations across users. 

After the mode-switching trials, each participant completed the experiments with one of the three data-driven mappings in a randomized single-blind set-up. We included an option to reset the robot to the initial pose. We did not impose any restrictions on when participants reset the robotic system. One strategy would be participants resetting the robot between each subtask instead of performing the movement end-to-end because of the testing environment set-up. This strategy was not necessarily optimal because no demonstrations were collected to perform the second and third sub-tasks from the start pose.

We collected several subjective metrics, including the Nasa-TLX score and a subjective experience, on a 7-point Likert scale. We include the Likert scale questions in the appendix. Figure~\ref{fig:nasaworkload} shows violin-plots of participants Nasa-TLX scores for each model. We found no statistical significance between the four models by the Dunn median statistical test for a p-value of 0.05. We include similar plots for the Likert scale results in the appendix, which only showed statistical significance on two pair-wise relations.

We furthermore measure success rates of the entire task and sub-tasks. Complete success with any system required picking up the cup, throwing contents into the red bin, and successfully throwing the cup away in the green bin. We show the percentage of trials for reported scenarios in Figure~\ref{fig:completion Rates}. We break down results into several scenarios based on our findings. The significant categories included full success or just missing pouring, which was the usual failure case. We counted partial successes where participants did some portions of the task successfully and otherwise reported failures as no subtask completed.

\subsubsection{Discussion}

Our user study results suggest that compared to the previously proposed SCL method and mode switching, users had similar experiences with both SCN and AE approaches. We found few statistical differences in our survey results or the Nasa-TLX score, suggesting users had similar experiences across mapping functions.   If we considered missing the pour as tolerable, the AE method could be considered objectively superior to SCN and SCL where users only struggled pouring with the AE model.  

However, in preliminary evaluations, we had to collect three data sets of similar movements and combine them for AE to learn a meaningful mapping. We found this decision to drastically impact its performance compared to other models. In contrast, we found that SCN and SCL were generally able to learn better mappings with any of the three datasets. This is an open challenge of our research as data-driven interfaces are sensitive to data quality and initialization. This randomness could also affect the mapping between training runs, and investigating the effects of this in deployment is a promising future work. 

We noticed a few intriguing observations throughout the experiments. Generally, users were "greedy" with their actions. They would prefer to stick to geometrically simple movements (e.g., moving directly towards the red bin) even if seemingly more complex movements would improve the chance of success. We observed this behavior despite explaining how the movements were generated and showing video footage of these movements. It might be a promising future direction to encourage users to produce movements that follow the demonstration data more closely instead of mapping to user preference \cite{Li2020LearningUM}. We also offered to help participants in the task if requested and found that users never took this option. Abstaining from assistance could be because they needed to remember about it. However, even in the few cases we reminded participants during trials of this option, they did not request assistance.

\section{Collaborative AI and Modeling Humans}
In this section, we briefly explain the connection of our work between collaborative artificial intelligence systems and human modeling and our interest in attending this bridge program. Our claim through the paper is that nonlinear model complexity provides limited performance gains in an AI teleoperation system. We derive our hypothesis from general mathematical results (i.e., the Taylor approximation) and show with human subjects that local linear action mapping is generally sufficient. We are attempting to distinguish the applications of mathematics in AI and their applications to human subjects, this venue's focal point of discussion.

However, a significant limitation of our work is the need for more expertise in modeling human behaviors. We motivated our work based on previous research in  data teleoperation or our own experience, which otherwise have limited discussion on human modeling literature. By attending the Collaborative AI and Modeling of Humans bridge program, we hope to gain insights from experts in human modeling to motivate future work. An interesting potential research direction is to better simulate human behaviors for teleoperation. A challenge faced in data-driven teleoperation is weak models of humans in simulation, which need to be better justified in the literature. 

As various hyperparameters can affect the performance of a data-driven teleoperation system, an automated way to correlate simulated human behavior to actual performance is imperative. We envision a path forward: apply research results to human modeling that may generate more accurate simulated teleoperation experiments, minimizing the need for extensive human user studies, which can be time-consuming. 
\section{Future work and Conclusion} \label{sec:conclusion}


In this work, we proposed several nonlinear action mapping models to compare against previously proposed local linear models. Based on our simulation results, SCN models work well at regressing actions, and although possible to build soft-reversible nonlinear action maps, more is needed to help the user experience further. Our user study results suggest that users showed no strong preference over one nonlinear action map to another once reverting actions were plausible. Results in simulation, subjective responses, and objective results add evidence to this. 

However, further research on developing the data-driven teleoperation paradigm is needed. From our experience, existing nonlinear architectures are sufficient, and research into other aspects of data-driven teleoperation is necessary. More research into architectures is warranted, but more rigorous standards for data teleoperation mapping comparisons are a crucial next step. Identifying a means to automate correlating the quality of action maps to user experience is a promising future work direction because the cost to conduct user studies is nontrivial for testing every change to a model. Action maps are a promising direction for software solutions to learn teleoperation mapping functions, but future work is needed for these methods to be useable in the real world.


\newpage
\bibliography{aaai24}

\newpage
\section{Appendix}
Our appendix contains additional mathematical and experimental results. These include a few additional details on the proposed state-conditioned nonlinear mappings (SCN), a theoretical section on extending soft-reversibility and some additional details on our user study. The information is helpful to understand our experimental results on comparing architectures for state conditioned action mappings. 
\section{Bounding Tensor Layer Lipschitz}
In this section we show our derivation for the Lipschitz upperbound of tensor layers. These changes were needed due to issues naively running the projection algorithm of \citet{gouk2020regularisation}. Algorithm~\ref{alg:algorithm} shows this function, and changes are in blue text. This function is called after each gradient update on the model parameters. 

One can rewrite a 3D tensor operation as a matrix product between $\bar{H} \in \mathbf{R}^{h \times wn}$ and matrix $A \in \mathbf{R}^{wn \times w}$. Matrix $\bar{H}$ is the reshaping of the 3D tensor along the dimension we multiply by actions, and $A$ is a block diagonal matrix of the input vector for the first tensor product operation. Research in neural adaptive control systems has applied the same operation \cite{dheeraj2019directadaptiveNCD}. Note that for two states $x$ and $x'$ we use $\Delta x = x - x'$ to derive the tensor Lipschitz relationship:
\begin{align*}
         ||\bar{H}Ax - \bar{H}Ax'||_p \leq & L||x - x'||_p &  \implies \\
||\bar{H}A(\Delta x)||_p\leq & L||\Delta x||_p  & \implies\\
 ||\bar{H}||_p ||A(\Delta x)||_p \leq & L||\Delta x||_p & \implies\\  
||\bar{H}||_p \frac{||A(\Delta x)||_p}{||\Delta x||_p} \leq & L & \implies\\
 ||\bar{H}||_p ||A||_p \leq & L . 
\end{align*} 
In the second to last step, we apply the definition of the operator norm on $\|A\|_p$. This derivation shows if we can constraint the matrix norms of the actions, which are user defined, and the weights to be less than the desired $L$ the tensor layer should be Lipschitz continuous. We use $p = 2$ in our experiments, but the infinity or L1 norm is also applicable.

Unfortunately, one must identify each layer input's maximum norm to apply the algorithm of \citet{gouk2020regularisation} on $\bar{H}$. The first layer input norm is the action's norm and is specified by the user. We assume the model should account for the maximal representation with bounded input values for the subsequent hidden layers. We take the sum over the absolute values of 3D tensor H and take the norm of the activations so that each layer input dimension $a_i = \sum_{jk} |H_{ijk}|$. This approach is a heuristic; we leave it as future work to constrain hypernetwork Lipschitz constants properly because it is nontrivial to address

\begin{algorithm}[tb]
\caption{Modified Projection Step of learning algorithm of \citet{gouk2020regularisation}.
Blue text are changes for SCN changes. }\label{alg:goukalg}
\label{alg:algorithm}
\textbf{Inputs}: 
\begin{itemize}
    \item  $\bar{H} \in \Theta$ parameters of model, $H \in R^{h \times wn}$ 
    \item $M = \argmax_{a \in A} ||a||$
    \item $\lambda$: desired layerwise Lipschitz 
    \item $\sigma$: next layers activation function 
\end{itemize}
\textbf{Output}: Spectral Projected Weights
\begin{algorithmic}[1] 
{\color{blue} \STATE A $\gets M$}
\FOR{$\bar{H}^{i} \in \Theta$}
\STATE $L \gets \frac{||\bar{H}||_p}{\lambda}$
\STATE $W \gets \frac{W}{\max(1,{\color{blue}A}L)}$
{\color{blue}
\STATE $H \gets \text{RESHAPE}(\bar{H})$ (3D tensor shape)
\STATE A $\gets \| \sigma(\sum_{jk} |H_{jk}|) \|$ (sum over input dimension)}  
\ENDFOR
\STATE \textbf{return} $\Theta$
\end{algorithmic}
\end{algorithm}

\section{Trajectory Reversibility}

This section provides bounds on trajectory soft reversibility, which accounts for the changing actions over several time steps. Our proof shows that the induced vector field with the inverse action sequence will return to a previously visited state or otherwise bound this distance. Our analysis is similar to the error analysis of the Euler method, and our proofs use standard results on solutions to initial value problems \cite{stuart2009dynamicalsystems}. 

We first show that in a continuous time system, any arbitrary function $f(x(t), a(t))$ can return to the exact initial state $x(0)$. This section defines the robotic state $x(t)$ as the robot's joint configurations. The actions $a(t)$ are inputs from a policy in a bounded symmetrical domain $a \in [-c, c]^{n}$ for some constant $c \in \mathbf{R}^{+}$. The action map follows the dynamics of Equation~\ref{eqn:dynamicalsystem}. $f$ is Lipschitz continuous for constant $L$, and is an odd function in actions (i.e. $f(x(t), -a(t)) = -f(x(t), a(t))$). 

\begin{theorem}[Trajectory Reversibility] \label{thm:trajectory_reversibility}
In the interval $0 \leq t \leq 1$ following $f$ in Equation~\ref{eqn:dynamicalsystem}, and assuming that $a(t+1) = -a(t-1)$,  we have $x(2- t) = x(t)$.
\end{theorem}

\begin{proof} Suppose $d(t) := x(t+1) - x(1-t)$, such that $d(0) = x(1) - x(1) = 0$. Let  $t_0 = \inf_{t \in [0, 1]} d(t) \neq 0$ be the first instance $d(t)$ deviates from  0, then $t_0 > 0$ because $d(0) = 0$, and that $d(t) = 0$ for all $t \leq t_0$. 
Take the derivative w.r.t.  $t$:
\begin{align*}
    \frac{\partial d(t)}{\partial t} &= \frac{\partial x(t +1)}{\partial t} + \frac{\partial x(t - 1)}{\partial t}\\
                                     &= f(x(t + 1), a(t+1 )) + f(x(1 - t), a(1 - t))\\
                                     &= f(x(1 - t), a(1 - t)) - f(x(t + 1), a(1 - t))
\end{align*}
where $f(x(t+1), a(1 + t)) = -f(x(t + 1), a(1-t))$ because we have $a(t+1) = -a(1 - t)$. This derivative is zero for $t=t_0$, because $d(t_0) = 0$, or that $d(t_0) = 0$ and $\dot d(t_0) = 0$. Moreover, for $d(t) = 0$, we have  $\partial d(t)/\partial t = f(q(t + 1), a(t+1 )) + f(q(1 - t), a(1 - t)) = f(q(t + 1),a(t+1)) - f(q(t + 1),a(t+1)) = 0 = \dot 0$.

By the Picard-Lindelof theorem  \cite{millODE2009} there exists $t >t_0$ such that $d(t) = 0$ is the unique solution of the differential equation for $t >t>t_0$.
\end{proof}

This result reverses the dynamics' gradient flow induced by the user actions' sequence. Even though the movements may change with a state-dependent action map, this result shows that a state-conditioned action map can perfectly return to a state when doing the opposite action sequence! The result further connects the accuracy of reversibility with the sampling rate. The faster we can update the velocities from an action map, the easier it should be to reverse actions.

\subsection{Reversing with Discretization Error }
Unfortunately, discretization errors cause reversed trajectories to deviate from any initial state. We account for discretization by bounding the distances from an initial state over a trajectory. We assume that $||f(s, a)|| \leq M$ where $M \in \mathbf{R}^+$. This is a reasonable assumption because robot joints exist in a bounded space, and a user's max action is manually specified. We derive a 
 necessary lemma to bound distance between subsequent states on the trajectory. 

\begin{lemma}
    For two subsequent points $x_T(k)$ and $x_T(k+1)$ following a function $f$ with Euler's method, we have that $||f(x_T(k), a) - f(x_T(k+1), a)|| \leq ML\nu$ where $M = \max_{x\in X,a \in A}||f(x, a)||$
\end{lemma}
\begin{proof}
    \begin{align*}
        ||f(x_T(k), a) - f(x_T(k+1, a)|| &\leq L||x_T(k) - x_T(k+1)||  \\
        & = L||\nu f(x(k), a)||  \\
        & \leq M L \nu
    \end{align*}
\end{proof}

This lemma states the maximal distance a user's action can take the robot's state at any point during control. Naturally, between any two points, the maximal distance cannot exceed the maximum velocity command times the sampling rate. We use this result to now state the bound on soft reversibility over trajectories: 
\begin{theorem}[Trajectory Soft Reversibility]
    Suppose there is an odd and bounded function $f(x(t), a(t))$ that is Lipschitz continuous $||f(x(t), a(t)) - f(x'(t), a(t))|| \leq L ||x(t) - x'(t)||$. For points $x_T(k) = x(0) + \sum_{i=0}^{k-1}f(x_T(i), a_T(i))$ such that $ a_T(T+k) = -a_T(T -k)$ for a fixed duration $T \in \mathbf{N}$, we have 
    \begin{equation}
        ||x_T(0) - x_T(2T)|| \leq \nu M [\exp^{TL\nu} - 1]
    \end{equation}
\end{theorem}
\begin{proof}
    Our proof is similar to error bound in Euler's method but exploits knowledge about the true function. Let $a_T(T + k) = -a_T(T- k)$, $E_t = ||x_T(T + t) - x_T(T - t)||$ and $\Delta f_t = || f(x_T(t + 1), a(t)) -  f(x_T(t), a(t))||$.
    \begin{align*}
        E_T &= ||x_T(0) - x_T(2T)|| \\
       & \leq ||x_T(1) - \nu f(x_T(1), a(1)) - x_T(2T-1) \\
       & \quad + \nu f(x_T(2T-1), a_T(1))|| + \nu \Delta f_0 \\
        &\leq E_{T-1} + \nu||f(x_T(2T - 1), a_T(1)) -  f(x_T(1), a_T(1))|| \\
        & + \nu \Delta f_{0}  \\
        & \leq E_{T-1} + L\nu||x_T(2T -1) - x_T(1)|| + \nu \Delta f_{0}  \\
        & = E_{T-1}(1 + L\nu) + \nu \Delta f_{0}   \\
        &\leq E_{T-1}(1 + L\nu) + ML\nu^2
    \end{align*}
    Where we use the previous lemma to bound subsequent steps. If we define $c = 1 + L \nu$ and $G = MLv^2$, then recursively applying the definition of $E_t$ leads to the following:
    \begin{align*}
        c E_{T-1} + G & \leq c (c E_{T- 2} + G) + G \\
                      & = c^2 E_{T-2} + (cG + G) \\
                      & \leq c^2 [c E_{T-3} + G] + (cG + G) \\
                      & = c^3 E_{T-3} + c^2G + cG + G  \\ 
                      & \leq ... \\
         E_T          &\leq (c)^{T-1} E_0 + G \sum_{i=0}^{T -1} c^i, 
    \end{align*}
    and replacing definitions of $c$ and $G$:
    \begin{equation*}
        E_T \leq (1+L \nu)^{T-1} E_0 + ML\nu^2 \sum_{i=0}^{T -1} (1 + L \nu)^i.
    \end{equation*}
 Note that $E_0 = 0$ because the optimal inverse solution and approximation start at same $x_T(T)$ position. Applying the geometric series on the second term we end up with:
\begin{align*}
         E_T &\leq ML\nu^2 (\frac{(1 + L\nu)^{T} - 1}{L \nu}) \\
            & = M \nu [ (1 + L\nu)^{T} - 1] \\
            & \leq \nu M [\exp^{TL\nu} - 1]
\end{align*}
\end{proof}
This proof shows that as the step size decreases (i.e., that we have a higher sampling rate), we expect to see the error rate decrease to zero.  

The previous proof assumes that the learned function $f$ is exactly reversible. We found experimentally that our loss function approach does not guarantee this is true. When probing the models, we would find the cosine similarity to usually be -1 even if there was a nonzero mean square error between the true $-f(x,a)$ and $f(x, -a)$. This observation suggests that the model predicts that predictions $f(x, a) $ and $f(x, -a)$ point in opposite directions, but the magnitudes vary slightly. To quantify the consequences of residual error, we have the following corollary. 

\begin{corollary}[Residual Inversing Actions]
    Suppose that $f(x, a)$ is such that $f(x, -a) = -f(x, a) + \epsilon$ is some bounded residual error $||\epsilon|| \leq E$. Then 
    \begin{equation}
        ||x_T(0) - x_T(2T)|| \leq (\nu M + \frac{E}{L}) [\exp^{TL\nu} - 1]
    \end{equation}
\end{corollary}

\begin{proof}
    The proof follows the same steps as the previous proof, but we have that:
    $\Delta f_t = ||f(X_T(t+ 1), a(t) - f(x_T(t), a(t)|| + ||\epsilon(t)||$.  where $\epsilon(t)$ is the error for taking $-a(t)$.  
    In the previous proof then would use the following relationship:
    \begin{equation*}
        \nu \Delta f_t \leq ML\nu^2 + \nu E 
    \end{equation*}
    
    Applying the above bound after telescoping we would end up with the following:
    \begin{align*}
         E_T &\leq (ML\nu^2 + E\nu) (\frac{(1 + L\nu)^{T} - 1}{L \nu}) \\
            & = (M \nu + \frac{E}{L}) [ (1 + L\nu)^{T} - 1] \\
            & \leq (M \nu + \frac{E}{L}) [[\exp^{TL\nu} - 1]
    \end{align*}
\end{proof}

At least in our analysis, this provides insight that a typical end-to-end encoder is inappropriate for teleoperation settings. Reducing the residual requires increasing model capacity and data coverage of the task interactions. This issue is mitigated directly by designing a neural network to encode reversibility in its structure, requiring no further parameter tuning to minimize residual errors.

\subsubsection{Soft Reversibility Experiments} 
We perform experiments to validate our theoretical results on soft reversibility. The hyperparameters for these models are otherwise the same as in the rest of our experiments, except that the tensor layers do not use matrix bias terms as described in the methods section. We generated random actions of unit length, which we executed for increasing durations with a fixed step size $\nu = 0.001$. We ran the reverse trajectory for the same duration as the forward trajectory and measured the gap between the final and initial start states. The models were run on a planar 5-DOF reaching task where trajectories each point a line in the X-Y plane. 

 Although we can enforce an upper bound on the Lipschitz constant, finding the actual value is an NP-hard problem \cite{gouk2020regularisation}. We thus report the specified hyperparameter upper bound used during training, the observed upper bound from the model, and the estimated maximal L2 norm of the input Jacobian. We estimate the Jacobian norm via gradient descent along the state space. Similarly, identifying the actual maximal value of the action map is difficult in practice as there is no analytic solution. We again turn to gradient descent to generate an approximation of the model's larger velocity magnitude. We applied the same procedure to estimate the maximal residual error of the AE-trained models.
 
Figure~\ref{fig:softrev-exp} shows results across different models. For each duration and sampling rate, we ran 20 trajectories and averaged the results, reporting the standard error. The bounds generally leave a large margin between the observed behavior and predicted performance. For \textit{SCN}, we noticed approximately a factor $\nu$ is between the actual and bound errors. In the AE case, one may notice that the error for $\lambda = 0.0156$ is looser than $\lambda = 0.1250$. This bound of the discrepancy is a product of a more significant residual $||E||$ and smaller Lipschitz constants, magnifying potential error effects. The model cannot accurately estimate the data and does not learn to reverse actions. 

Furthermore, the astute reader may notice that the estimated Jacobian norm of SCN violates the estimated upper bound of the model. During training, we used a truncated value of the actual action norm $||a|| = \sqrt{2}$, which could influence the estimation slightly. Interestingly, our upper bound estimates reflect the observed experimental results. The fact that the global maximal Jacobian norm disagrees with our Lipschitz estimation may mean that our model meets the upper bound locally. This suggests that additional theoretical results might be possible with looser assumptions. Although this requires further investigation, we consider this future work. Neural network Lipschitz constants are a nontrivial topic of discussion, and meaningful experiments would go beyond the teleoperation setting into other machine learning areas.

\begingroup
\begin{figure}[htp] 

    \begin{minipage}[t]{0.45\textwidth}
    \centering
        \includegraphics[trim=0mm 0mm 0mm 0mm, clip=true, width=\textwidth]{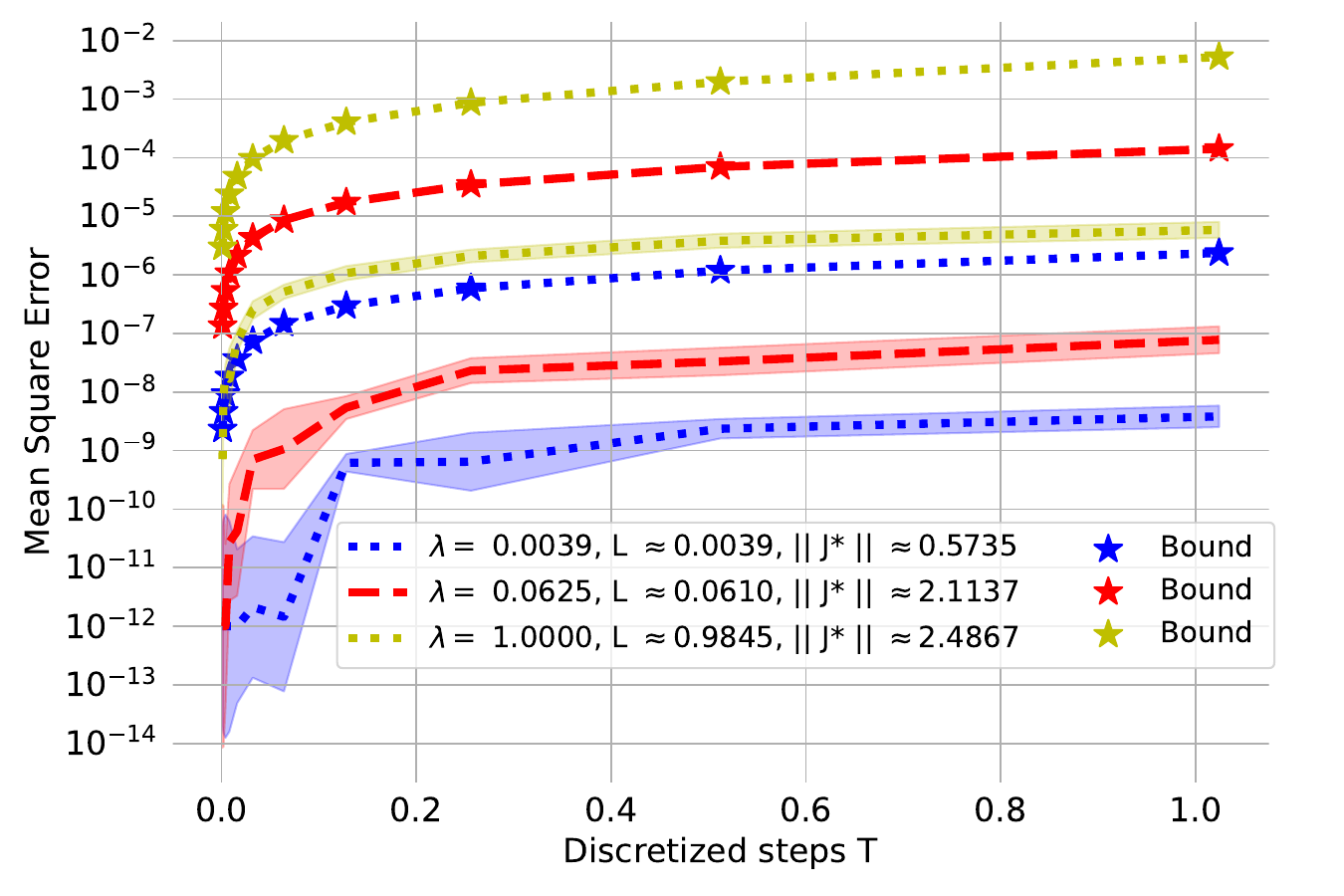}
        \caption*{State Conditioned Nonlinear Maps Regularize Action Weights}
        \label{fig:softrev-tensor}
    \end{minipage}
    \begin{minipage}[t]{0.45\textwidth}
        \centering
        \includegraphics[trim=0mm 0mm 0mm 0mm, clip=true, width=\textwidth]{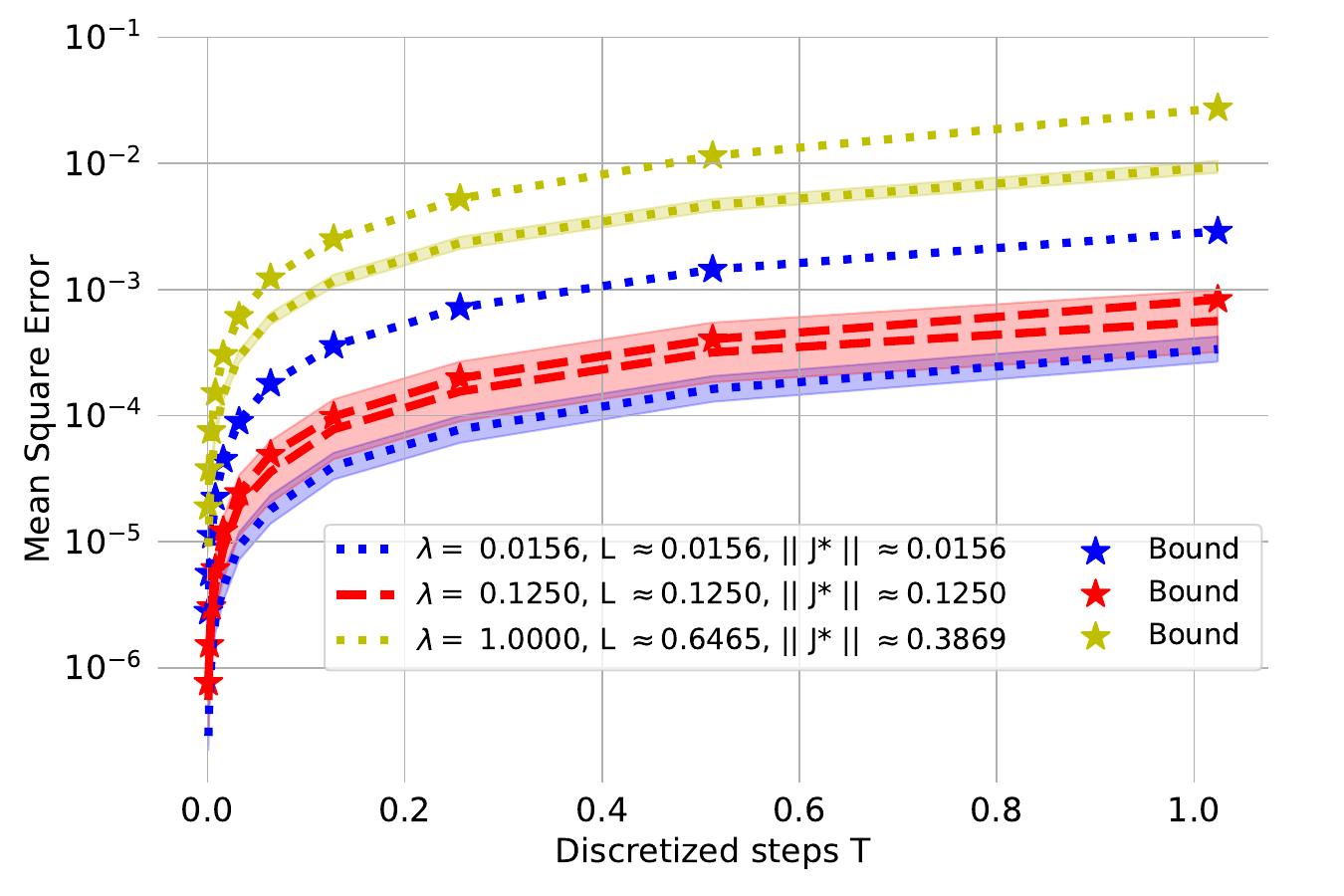}
        \caption*{Regularized Autoencoder}
        \label{fig:softrev-ae}
    \end{minipage}
\caption{Soft Reversibility Results with Autoencoder and SCN. We report the training Lipschits $\Lambda$ the estimated Lipschitz following regularizer method ($L$), and estimated maximum Jacobian over model $||J*||$. }
\label{fig:softrev-exp}
\end{figure}
\endgroup
\section{User Study Addtional Information}
This section briefly covers additional details to understand our user study experiments. These include how the training data was collected and how we mapped our mode-switching algorithm to the interfaced used in experiments. We show Likert Scale results in Figure~\ref{fig:pouringexperiments}

\subsection{Kinesthetic Demonstration}
We collected data using a Kinesthetic demonstration process. We divided trajectories into two sub-trajectories. In the first sub-trajectory, we would position the robot to grab the blue cup and end the demonstration. Starting from the position to grab the blue cup, we would close the gripper and then, in a single movement, pour the contents into the red box and position the robot arm over the green box. We never demonstrated the opening motion for the cup. We used an exponential moving average filter ($\gamma = 0.2$ weight) on the sensor joint positions to account for noise in the demonstration. We then took every third sample and calculated the finite difference between them. We used the finite difference approximation as the robot's input actions.

\subsection{Joy Stick Mapping}

In this section, we briefly discuss our joystick mapping. We used a \textit{Logitech G Extreme 3D Pro Joystick} for our experiments. When participants used mode switching, we mapped the first three modes to control different aspects of the cartesian space. For translation, the first mode controlled the X-Y plane concerning the table, mode 1 controlled the X-Z plane, and mode 2 controlled the Y-Z plane. Mode 3 controlled yaw and pitch, mode four controlled pitch and roll, and Mode 6 controlled yaw and roll. We define the modes listed concerning variable names in our code.

\begin{figure}[htp]
    \centering
    \includegraphics[width=0.5\textwidth]{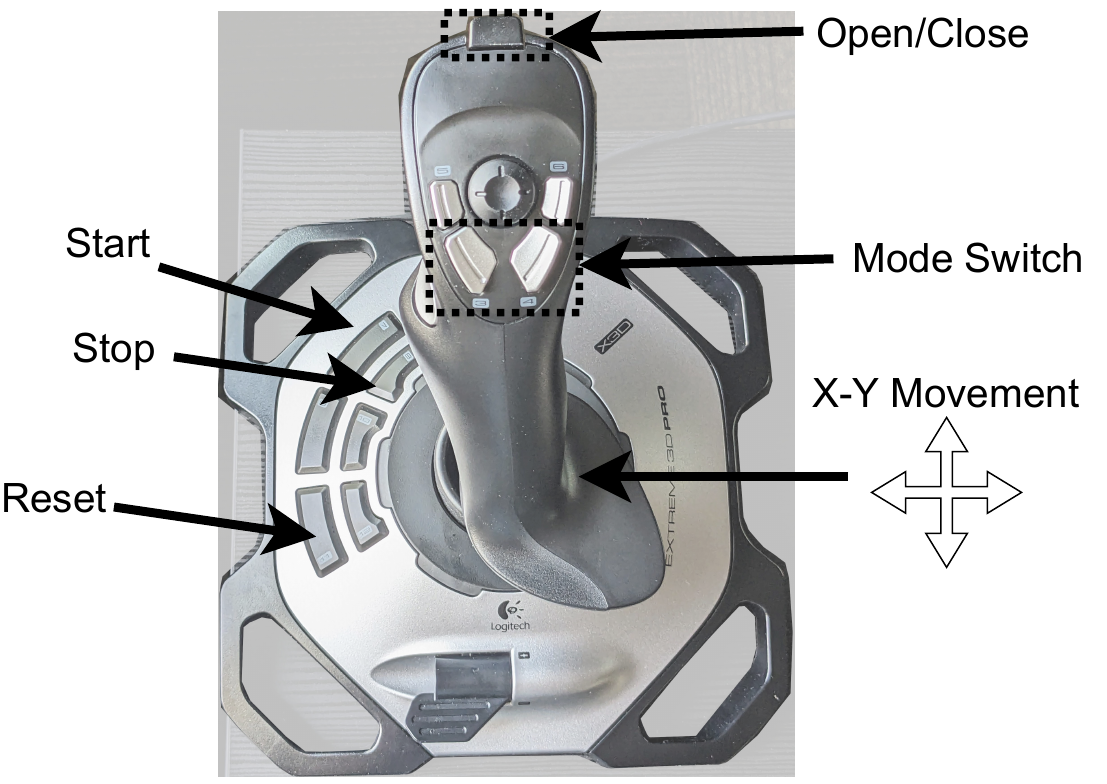}
    \caption{Logitech G Extreme 3D Pro Joystick. We visualize button mappings used in experiments}
    \label{fig:joystick-label}
\end{figure}

\begingroup
\begin{figure*}[htp] 
\vspace{0.5in}
\hspace{0.00in}
    \begin{minipage}[t]{0.35\textwidth}
        \centering
        \includegraphics[trim=0mm 0mm 0mm 7mm, clip=true, width=\textwidth]{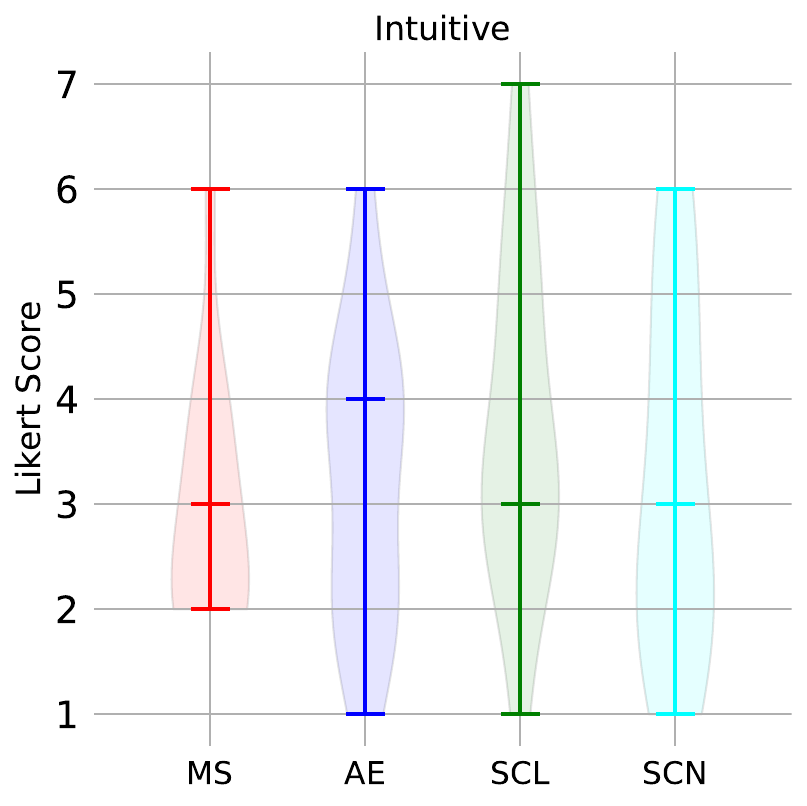}
        \caption*{Whether system controls were \textit{intuitive} }
        \label{fig:likert_intuitive}
    \end{minipage}
    \begin{minipage}[t]{0.35\textwidth}
    \centering
        \includegraphics[trim=0mm 0mm 0mm 7mm, clip=true, width=\textwidth]{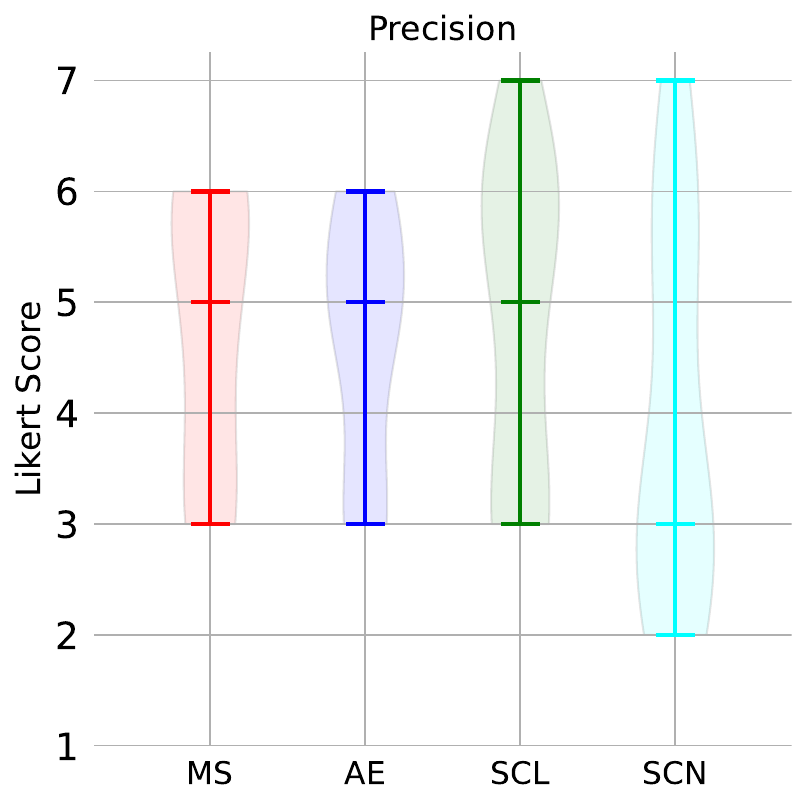}
        \caption*{How \textit{precise} commands were}
        \label{fig:likert_precise}
    \end{minipage}
    \begin{minipage}[t]{0.35\textwidth}
    \centering
        \includegraphics[trim=0mm 0mm 0mm 7mm, clip=true, width=\textwidth]{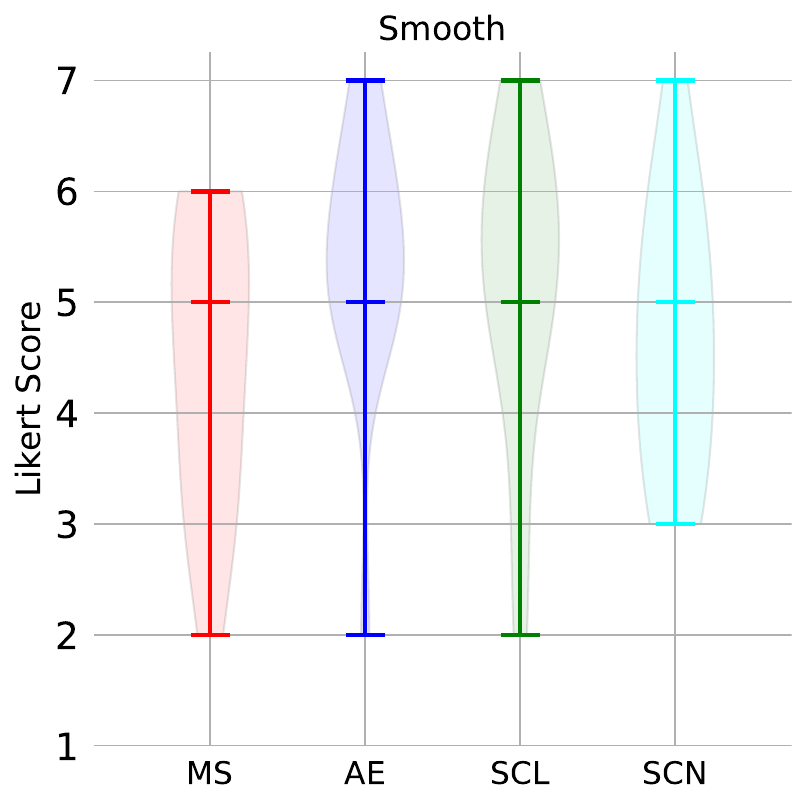}
        \caption*{\textit{Smoothness} between movements }
        \label{fig:likert_smooth}
    \end{minipage}
    \begin{minipage}[t]{0.35\textwidth}
    \centering
        \includegraphics[trim=0mm 0mm 0mm 7mm, clip=true, width=\textwidth]{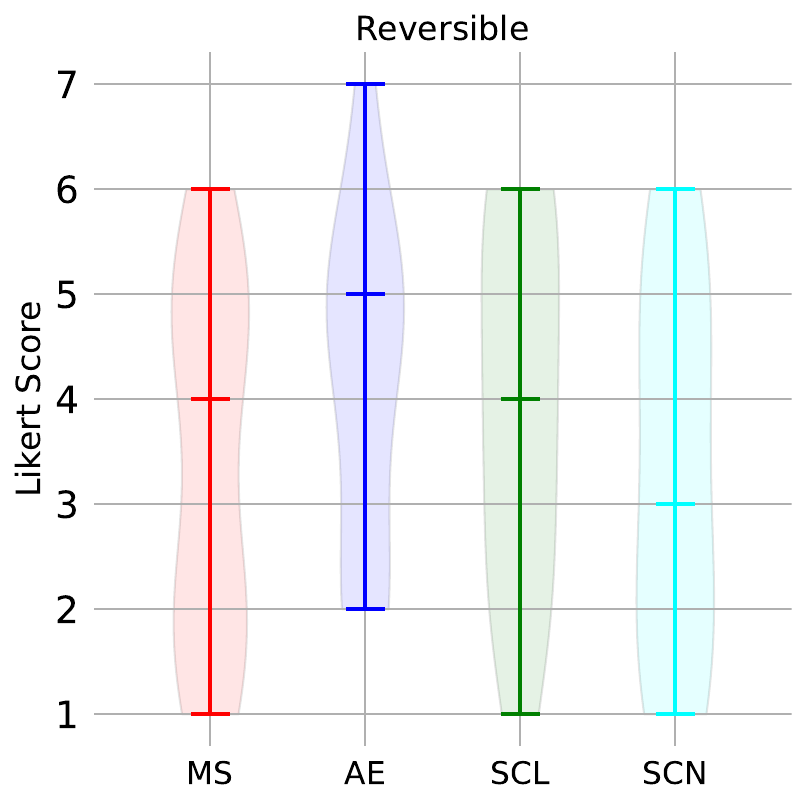}
        \caption*{Ability to \textit{Undo} movements}
        \label{fig:likert_undo}
    \end{minipage}
    \begin{minipage}[t]{0.35\textwidth}
    \centering
        \includegraphics[trim=0mm 0mm 0mm 7mm, clip=true, width=\textwidth]{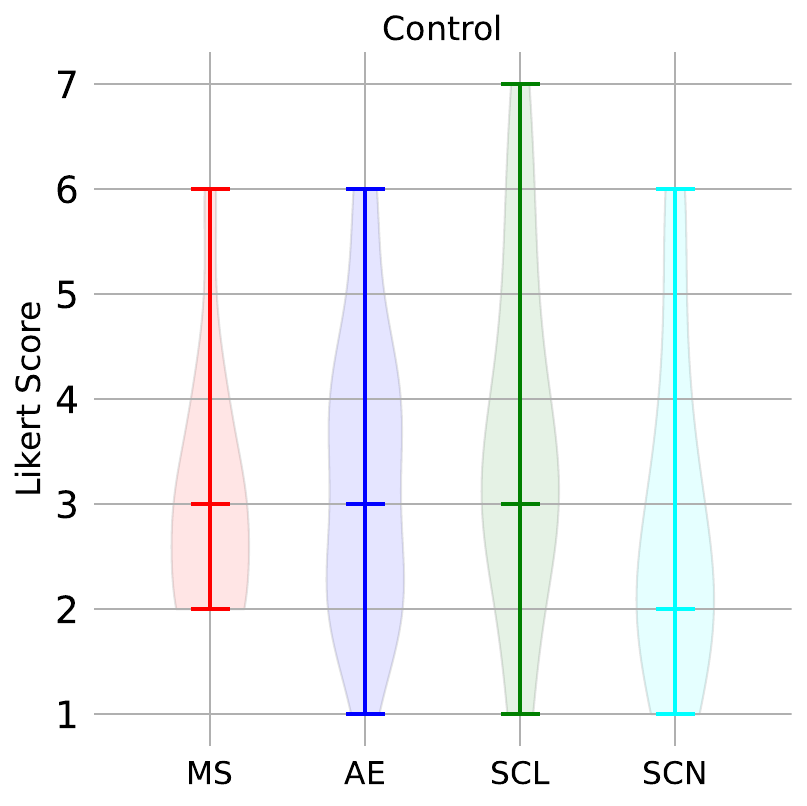}
        \caption*{Feeling in \textit{Control} of system }
        \label{fig:likert_control}
    \end{minipage}
\caption{Subjective Likert score results from User study. Statistical significance (p-value $<$ 0.05) was found between SCL and Tensor for Control and Smoothness between Mode switching and autoencoders. All other pairs were not significant. }
\label{fig:pouringexperiments}
\end{figure*}
\endgroup

\end{document}